\documentclass[11pt,letterpaper]{article}

\usepackage{amsmath,amssymb,amsthm}
\usepackage{mathtools}
\usepackage{hyperref}
\usepackage[margin=1in]{geometry}
\usepackage{booktabs}
\usepackage{graphicx}
\usepackage{xcolor}
\usepackage{lineno}

\newtheorem{theorem}{Theorem}[section]
\newtheorem{proposition}[theorem]{Proposition}
\theoremstyle{definition}

\theoremstyle{remark}

\DeclareMathOperator{\rank}{rank}
\DeclareMathOperator{\softmax}{softmax}
\DeclareMathOperator{\MAD}{MAD}
\DeclareMathOperator{\Var}{Var}
\newcommand{\R}{\mathbb{R}}
\newcommand{\E}{\mathbb{E}}

\title{Compressible Softmax-Attended Language\\
under Incompressible Attention}
\author{
Wonsuk Lee\\
Seoul National University\ \& SK Hynix America\\
\texttt{\small wonsuk.lee@snu.ac.kr}
}
\date{April 5, 2026}

\begin{document}
\maketitle

\begin{abstract}

Softmax attention defines an interaction through $d_h$ head                                                                             
dimensions, but not all dimensions carry equal weight once real
text passes through.  We decompose the attention logit field into                                                                          
a learned component and a generated component and measure their                                                                            
spectra separately.  For all 5,888 KV heads in five transformer                                                                            
language models (124M--7B parameters, four architecture families),
the logit energy field $\tilde{E}$ reaches 90\% of its variance                                                                             
in 2--11 singular components.  The learned interaction matrix
$W_Q^\mathrm{T} W_K$ needs 38--75 components for the same                                                                                  
threshold out of $d_h \in {64, 128}$.  The spectral gap is                                                                                 
5--25$\times$ in effective rank.  The compressibility of
softmax-attended language is a property of the data, not the                                                                               
frame that analyzes it. 

\end{abstract}

\vspace{1.0em}

\section{Introduction}
\label{sec:intro}

The key--value cache dominates the memory cost of autoregressive
transformer inference and grows linearly with context length.
Reducing it requires exploitable structure in the attention
pattern.

The row-centered attention logit energy field $\tilde{E}$ has rank
at most $d_h + 1$, where $d_h$ is the head dimension~\cite{lee2026invariants}, 
and its effective rank is much lower in practice.
The $d_h$-dimensional key cache is redundant.  
The question is where the redundancy originates.

Two possibilities.  First, the weight matrices $W_Q$ and $W_K$
impose spectral structure on the interaction, making the mechanism
itself low-rank.  A fixed projection from the model parameters
then compresses the key cache for any input.  Second, the input
data concentrates the interaction into a low-dimensional subspace
even though the weights are spectrally uniform.

Existing KV-cache compression methods assume one or the other.
Low-rank weight factorization and static head pruning treat the
redundancy as architectural.  Token eviction and attention-aware
merging treat it as input-dependent.  Neither tests the
assumption.  Comparing the two spectra answers the question and
predicts which class of methods should work.

We measure both spectra across five models in four architecture
families (GPT-2, LLaMA, Qwen, Mistral).  They differ by
$5$--$25\times$ in effective rank.  The \emph{learned} matrix
$M = W_Q^\mathrm{T} W_K$ distributes its variance uniformly across
all $d_h$ directions in every model tested.  The \emph{generated}
matrix $\tilde{E}$ concentrates almost all variance into a few.

The compression lives in the data, not the weights.  Effective key
cache compression requires data-adaptive projections that change
with every context.  The compressibility of softmax-attended
language is a property of how language activates attention, not of
how attention is built.

\section{Setup}
\label{sec:setup}

A single attention head~\cite{vaswani2017attention} maps $L$
input embeddings
$x_0, \ldots, x_{L-1} \in \R^{d_{\text{model}}}$ to queries and
keys through learned weight matrices:
\begin{equation}\label{eq:qk}
  q_i = W_Q\, x_i \in \R^{d_h}, \qquad
  k_j = W_K\, x_j \in \R^{d_h}.
\end{equation}
The attention logit is $Z_{ij} = q_i^\mathrm{T} k_j / \sqrt{d_h}$, and
the attention probability is $p_{ij} = \softmax_j(Z_{ij})$.

\subsection{The logit energy field (\emph{generated})}
\label{sec:logit-field}

Row-centering the logit matrix removes the position-dependent
baseline:
\begin{equation}\label{eq:etilde}
  \tilde{E}_{ij} = Z_{ij} - \bar{Z}_i,
  \qquad \bar{Z}_i = \frac{1}{L}\sum_{j=0}^{L-1} Z_{ij}.
\end{equation}
Since softmax is shift-invariant,
$p_{ij} = \softmax_j(\tilde{E}_{ij})$.  The matrix $\tilde{E}$
carries the same information as $Z$ for computing attention.

The row-sum identity $\sum_j \tilde{E}_{ij} = 0$ holds by
construction. Therefore $\tilde{E}\,\mathbf{1} = 0$ and $\mathbf{1}$ is
in the right null space.  Since $Z = Q K^\mathrm{T} / \sqrt{d_h}$ has rank at most $d_h$,
and row-centering adds a rank-one term,
$\rank(\tilde{E}) \le d_h + 1$.

The SVD of $\tilde{E}$, with rank $R \le d_h + 1$,
\begin{equation}\label{eq:etilde-svd}
  \tilde{E} = \sum_{k=1}^{R} \sigma_k\, u_k\, v_k^\mathrm{T},
\end{equation}
decomposes the interaction into $R$ channels ordered by
variance~\cite{golub2013matrix}.  The $k$-th channel has strength
$\sigma_k$, left singular vector $u_k \in \R^L$ over queries, and
right singular vector $v_k \in \R^L$ over keys.  If the singular
values decay rapidly, then a few channels approximate the full
matrix and the attention pattern is compressible.

The right singular vectors satisfy $v_k \perp \mathbf{1}$.  
This follows from $\tilde{E}\,\mathbf{1} = 0$.  The vector
$\mathbf{1}$ is an eigenvector of
$\tilde{E}^\mathrm{T} \tilde{E}$ with eigenvalue zero, while
each $v_k$ has nonzero eigenvalue $\sigma_k^2$.  Eigenvectors of
a symmetric matrix with distinct eigenvalues are orthogonal.

\subsection{The interaction matrix (\emph{learned})}
\label{sec:interaction}

The logit is a bilinear form in the input embeddings:
\begin{equation}\label{eq:bilinear}
  Z_{ij} = \frac{x_i^\mathrm{T} M\, x_j}{\sqrt{d_h}},
  \qquad \text{where} \quad M = W_Q^\mathrm{T} W_K
    \in \R^{d_{\text{model}} \times d_{\text{model}}}.
\end{equation}
The matrix $M$ factors through $\R^{d_h}$
($W_Q \in \R^{d_h \times d_{\text{model}}}$,
$W_K \in \R^{d_h \times d_{\text{model}}}$), so
$\rank(M) \le d_h$.  Its SVD,
\begin{equation}\label{eq:M-svd}
  M = \sum_{k=1}^{d_h} \lambda_k\, p_k\, r_k^\mathrm{T},
\end{equation}
decomposes the mechanism's capacity into $d_h$ channels. 
The singular values $\lambda_k$ and their decay rate measure 
how much spectral structure the weights alone impose.

The generated spectrum reflects both the weights and the input.
The learned spectrum reflects the weights alone.  Comparing the
two reveals whether the compression is architectural or
input-driven.

\section{Experiments}
\label{sec:experiments}

We computed the singular value spectra of $\tilde{E}$ and $M$ for
every attention head in five models spanning four architecture
families and two head dimensions.

\paragraph{Models.}
GPT-2 (124M, $d_h\!=\!64$, 144 heads).
LLaMA-3.2-1B (1.2B, $d_h\!=\!64$, 512 heads).
LLaMA-3.2-3B (3.2B, $d_h\!=\!128$, 672 heads).
Qwen-2.5-3B (3B, $d_h\!=\!128$, 576 heads).
Mistral-7B (7B, $d_h\!=\!128$, 1024 heads).
The three larger models use grouped query attention, where
multiple query heads share a single key-value head.

\vspace{-0.6em}

\paragraph{Protocol.}
For $M$, we compute singular values directly from the weight
matrices via QR factorization.  For $\tilde{E}$, we extract $Q$
and $K$ (with RoPE applied where applicable) on five texts from
the Gutenberg Project~\cite{gutenberg} (Dickens, Darwin,
Shakespeare, the King James Bible, and Adam Smith) at $L = 256$
and compute $Z = Q K^\mathrm{T} / \sqrt{d_h}$.  All SVDs use
float64 arithmetic.  Tables report the median across all heads
pooled over the five texts.  The effective rank at 90\% varies by
at most $\pm 1$ across texts (standard deviation $\le 1.2$).

\vspace{-0.6em}

\paragraph{Results.}
Table~\ref{tab:cumvar} compares the cumulative fraction of
variance ($\sum_{k=1}^r \sigma_k^2 / \|\tilde{E}\|_F^2$)
captured by the top $r$ components. 
Table~\ref{tab:effrank} translates these fractions into effective
rank.

\vspace{-0.5em}

\begin{table}[h!]
\centering
\caption{Cumulative variance captured by the top $r$ singular
components (median across all heads).  $\tilde{E}$ concentrates
its variance in a few components. $W_Q^\mathrm{T} W_K$ does not.
Models are grouped by head dimension: $d_h = 64$ (left) and
$d_h = 128$ (right).}
\label{tab:cumvar}
\small
\begin{tabular}{r cc cc cc cc cc}
\addlinespace[1.0em]
\toprule
& \multicolumn{2}{c}{GPT-2}
& \multicolumn{2}{c}{LLaMA-1B}
& \multicolumn{2}{c}{LLaMA-3B}
& \multicolumn{2}{c}{Qwen-3B}
& \multicolumn{2}{c}{Mistral-7B} \\
\cmidrule(lr){2-3} \cmidrule(lr){4-5}
\cmidrule(lr){6-7} \cmidrule(lr){8-9} \cmidrule(lr){10-11}
$r$ & $M$ & $\tilde{E}$
    & $M$ & $\tilde{E}$
    & $M$ & $\tilde{E}$
    & $M$ & $\tilde{E}$
    & $M$ & $\tilde{E}$ \\
\midrule
  1 &   5\% &  72\% &   6\% & 43\%
    &   4\% & 42\%  &   4\% & 41\%
    &   4\% & 46\% \\
  2 &  10\% &  90\% &  12\% & 63\%
    &   7\% & 61\%  &   7\% & 60\%
    &   8\% & 64\% \\
  5 &  20\% &  96\% &  25\% & 83\%
    &  15\% & 81\%  &  16\% & 80\%
    &  17\% & 82\% \\
 10 &  33\% &  98\% &  42\% & 93\%
    &  25\% & 90\%  &  27\% & 90\%
    &  29\% & 91\% \\
 20 &  53\% &  99\% &  66\% & 98\%
    &  42\% & 96\%  &  45\% & 96\%
    &  47\% & 96\% \\
 40 &  81\% & 100\% &  92\% & 100\%
    &  66\% & 99\%  &  70\% & 99\%
    &  72\% & 99\% \\
\bottomrule
\end{tabular}
\end{table}
\vspace{-0.2em}

\begin{table}[h!]
\centering
\caption{Effective rank: number of singular components for the
stated variance threshold (median across all heads).  The spectral
gap between $M$ and $\tilde{E}$ persists across all five models.}
\label{tab:effrank}
\small
\begin{tabular}{r cc cc cc cc cc}
\addlinespace[1.0em]
\toprule
& \multicolumn{2}{c}{GPT-2}
& \multicolumn{2}{c}{LLaMA-1B}
& \multicolumn{2}{c}{LLaMA-3B}
& \multicolumn{2}{c}{Qwen-3B}
& \multicolumn{2}{c}{Mistral-7B} \\
\cmidrule(lr){2-3} \cmidrule(lr){4-5}
\cmidrule(lr){6-7} \cmidrule(lr){8-9} \cmidrule(lr){10-11}
Threshold & $M$ & $\tilde{E}$
          & $M$ & $\tilde{E}$
          & $M$ & $\tilde{E}$
          & $M$ & $\tilde{E}$
          & $M$ & $\tilde{E}$ \\
\midrule
80\% & 40 &  2 & 29 &  5 & 57 &  5 & 53 &  5 & 50 &  5 \\
90\% & 49 &  2 & 38 &  8 & 75 & 11 & 70 & 11 & 66 & 10 \\
95\% & 55 &  4 & 44 & 13 & 88 & 18 & 83 & 17 & 80 & 17 \\
99\% & 61 & 18 & 54 & 23 & 107 & 39 & 104 & 40 & 102 & 38 \\
\bottomrule
\end{tabular}
\end{table}
\vspace{0.4cm}

At 90\% variance, $\tilde{E}$ needs 2--11 components; $M$ needs
38--75.  The ratio ranges from $5\times$ (LLaMA-1B) to
$25\times$ (GPT-2).  The three $d_h = 128$ models converge to
$\tilde{E}$ effective rank 10--11 despite learned effective ranks
of 66--75, a gap of $6$--$7\times$.  The learned spectrum is flat
in every case.

\section{Conjecture}
\label{sec:conjecture}

\subsection{Incompressible attention framework}
\label{sec:incompressible}

The flat spectrum of $M$ implies that no fixed, low-rank projection
can capture most of the interaction.  Discarding any subset of the
$d_h$ directions loses a proportional fraction of the mechanism's
capacity.  The attention framework is, in spectral terms,
incompressible.

If $M$ had concentrated its spectrum, compression would be
straightforward.  Project queries and keys onto the dominant
eigenvectors, discard the rest, and reduce the head dimension from
$d_h$ to $r \ll d_h$.  The flat spectrum rules this out.  Each head
retains roughly equal capacity in all $d_h$ directions, able to
represent whichever query--key interaction the input demands.

\subsection{Compressible logit energy field}
\label{sec:compressible}

While $M$ describes what the mechanism can do, $\tilde{E}$
describes what it actually does on a specific input.  The sharp
spectral decay of $\tilde{E}$ means the query--key interaction
concentrates into a few dominant patterns for any given context.

The concentration comes from the data.  Layer normalization and
preceding transformer blocks shape the input embeddings $x_i$
into a context-dependent manifold rather than spanning all of
$\R^{d_{\text{model}}}$. Queries and keys inherit this 
low-dimensional structure. The matrix $Z = Q K^\mathrm{T} / \sqrt{d_h}$ 
has formal rank $d_h$ but much lower effective rank, 
determined by the data manifold. The empirical regularity 
$\mu_K = O(1)$ (key incoherence), measured across 16 transformer 
language models~\cite{lee2026invariants}, reflects the same geometry.  
The weights permit large $\mu_K$. 
Language does not require it. 
Training preserves $\mu_K \approx 2$.

Two properties of the data support the pattern's generality.
First, it holds across both $d_h = 64$ and $d_h = 128$.  Doubling
the head dimension roughly doubles the learned effective rank but
leaves the generated effective rank at 10--11.  Second, it is
stable across five texts from different centuries and genres
(standard deviation $\le 1.2$).

\subsection{Low-rank structure survives softmax}
\label{sec:softmax-bound}

Whether the sharp generated spectrum survives softmax is not
obvious.  A rank-$r$ approximation $\tilde{E}_r$ captures most of
the Frobenius-norm variance, but softmax is nonlinear and
amplifies residuals at peak-attention positions.

\begin{proposition}[Softmax stability of low-rank
  approximation]\label{prop:softmax}
Let $\tilde{E} \in \R^{L \times L}$ be the row-centered logit
matrix with SVD
$\tilde{E} = \sum_{k=1}^{R} \sigma_k\, u_k\, v_k^\mathrm{T}$,
and let $\tilde{E}_r = \sum_{k=1}^{r} \sigma_k\, u_k\,
v_k^\mathrm{T}$ be its rank-$r$ truncation.  Let
$p_i = \softmax(\tilde{E}_i)$ and
$p_i^{(r)} = \softmax((\tilde{E}_r)_i)$ be the attention
distributions in row~$i$.  If the singular vectors satisfy the
delocalization condition
\begin{equation}\label{eq:ipr-bound}
\max_j\, u_j^2 \le \frac{\beta}{\sqrt{L}}, \qquad
\max_j\, v_j^2 \le \frac{\beta}{\sqrt{L}}
\end{equation}
for a parameter $\beta$, then for every row~$i$:
\begin{equation}\label{eq:softmax-bound}
\bigl\|p_i - p_i^{(r)}\bigr\|_1
\;\le\;
\frac{\beta}{\sqrt{L}}\,\sum_{k=r+1}^{R} \sigma_k\,.
\end{equation}
\end{proposition}

\begin{proof}
The residual $\Delta_{ij} = \sum_{k=r+1}^{R} \sigma_k\,
(u_k)_i\, (v_k)_j$ satisfies
\[
|\Delta_{ij}|
\;\le\; \sum_{k=r+1}^{R} \sigma_k\, |(u_k)_i|\, |(v_k)_j|
\;\le\; \frac{\beta}{\sqrt{L}} \sum_{k=r+1}^{R} \sigma_k
\]
by the delocalization condition~\eqref{eq:ipr-bound}, since
$|(u_k)_i|\, |(v_k)_j| \le \beta / \sqrt{L}$ for all $i, j, k$.
This bounds $\|\Delta_i\|_\infty$ uniformly over rows.
Applying Theorem~\ref{thm:softmax-lip} (Appendix~\ref{app:appendix})
with $a = \tilde{E}_i$ and $b = (\tilde{E}_r)_i$,
\[
\|p_i - p_i^{(r)}\|_1
\;=\; \|\softmax(\tilde{E}_i) - \softmax((\tilde{E}_r)_i)\|_1
\;\le\; \|\Delta_i\|_\infty
\;\le\; \frac{\beta}{\sqrt{L}}\,\sum_{k=r+1}^{R} \sigma_k\,.
\qedhere
\]
\end{proof}

The delocalization condition~\eqref{eq:ipr-bound} is verified
directly.  Measuring $\beta = \max_j v_j^2 \cdot \sqrt{L}$ across
every singular vector of every head, from five models and
five texts at $L = 256$, gives median $\beta = 2.5$--$5.6$ and
maximum $\beta \le 14$.  The value does not grow with model size
or head dimension.  The delocalization lemma of~\cite{lee2026invariants} 
shows that $\beta$ is controlled by the key incoherence $\mu_K$ 
and the condition number $\kappa(K)$, 
both of which are bounded across trained models.

The bound shrinks with both context length (through $1/\sqrt{L}$)
and rank (through the tail $\sum_{k>r} \sigma_k$).  At $r = 20$,
the tail accounts for 2\% of the total variance.  The empirical
evidence at $L = 64$--$1{,}024$ shows no growth in the tail.

Direct measurement confirms the prediction.  At $L = 256$ across
five models and five texts, the median
$\|p_i - p_i^{(r)}\|_1$ is 0.31--0.46 (mean per row) at
$r = 10$, 0.18--0.29 at $r = 20$, and 0.05--0.14 at $r = 40$.
The worst row follows the same trend: 0.90--1.31, 0.52--0.84,
and 0.16--0.39 respectively.

\subsection{Conjecture statement}
\label{sec:conjecture_statement}

We conjecture that \emph{the low effective rank of $\tilde{E}$ is
a universal property of natural language processed by trained
transformers.}  The source is the low intrinsic dimensionality of
contextualized embeddings, not any spectral constraint in the
attention weights.  Models across four architecture families
demonstrate the pattern. Whether it extends beyond language remains open.

\section{Discussion}
\label{sec:discussion}

The spectral gap between $M$ and $\tilde{E}$ has consequences
for both compression practice and architectural understanding.

The flat spectrum of $W_Q^\mathrm{T} W_K$ rules out weight-based
compression.  Fixed, input-independent projections cannot exploit
the concentration in $\tilde{E}$ because the low-rank projection
that captures nearly all of the attention variance changes with
every context.  Effective compression requires data-adaptive
methods.

This separates existing KV-cache work into two classes.
Input-independent methods (low-rank weight factorization, static
head dimension reduction) compress the mechanism, which is already
spectrally uniform.  Context-adaptive methods (eviction policies,
token merging, attention-aware pruning) can exploit the sharp data
spectrum, but they must track a projection that shifts as the
context grows during autoregressive generation.

The gap between the mechanism's full-rank frame and the data's
low-rank activation serves a purpose.  Each head allocates
uniform spectral capacity across $d_h$ dimensions, yet language
fills only a few at any given moment.  The unused dimensions
allow the same head to serve diverse
contexts across the training distribution.  Idle capacity is
flexibility.

Proposition~\ref{prop:softmax} bounds per-head, per-layer
attention fidelity.  Three quantities remain open:
\begin{enumerate}
\item The output error
  $\|o_i - o_i^{(r)}\| = \|\sum_j (p_{ij} - p_{ij}^{(r)}) v_j\|$
  depends on the value matrix $V$, not just the attention weights.
\item Errors compound across layers.  Twenty layers each
  contributing a small $\ell_1$ attention error could cancel or
  amplify depending on whether they are correlated.
\item The accumulated effect on end-to-end perplexity remains an
  open experimental question.
\end{enumerate}

The spectral gap reported here is measured on language models.
If structured sequential input drives the concentration, vision
transformers and protein models would show similar patterns. If
it is specific to natural language, other modalities would exhibit
higher effective rank in $\tilde{E}$.  Either outcome would
clarify what makes attention compressible.

\bibliographystyle{plain}

\appendix

\section{Appendix}
\label{app:appendix}

\begin{theorem}[Softmax Lipschitz bound]\label{thm:softmax-lip}
For any $a, b \in \R^k$,
\begin{equation}\label{eq:softmax-lip}
\|\softmax(a) - \softmax(b)\|_1
\;\le\; \|a - b\|_\infty.
\end{equation}
\end{theorem}

\begin{proof}
Let $c = b - a$ and define $p^{(t)} = \softmax(a + tc)$ for
$t \in [0,1]$, so $p^{(0)} = \softmax(a)$ and
$p^{(1)} = \softmax(b)$.  By the fundamental theorem of calculus,
\[
\softmax(b)_i - \softmax(a)_i
= \int_0^1 \frac{d}{dt}\, p_i^{(t)}\, dt.
\]
The softmax Jacobian gives
$\frac{d}{dt}\, p_i^{(t)} = p_i^{(t)}(c_i - \mu_t)$,
where $\mu_t = \sum_j p_j^{(t)} c_j = \E_{p^{(t)}}[c]$.
Applying the triangle inequality inside the integral,
\[
\|\softmax(b) - \softmax(a)\|_1
\;\le\; \int_0^1
  \sum_i p_i^{(t)} |c_i - \mu_t|\, dt.
\]
The integrand is the mean absolute deviation
$\MAD(c;\, q) = \E_q[|c - \E_q[c]|]$ evaluated at
$q = p^{(t)}$.  Fix $t$ and write $q = p^{(t)}$,
$\mu = \mu_t$.  By Jensen's inequality
($\varphi(x) = x^2$ is convex),
$\MAD(c;\,q)^2 \le \Var_q(c)$.
By Popoviciu's inequality, for $c$ bounded in $[m, M]$,
$\Var_q(c) \le (M - m)^2 / 4$.
Combining,
\[
\MAD(c;\, q)
\;\le\; \sqrt{\Var_q(c)}
\;\le\; \frac{M - m}{2}
\;=\; \frac{\max_i c_i - \min_i c_i}{2}.
\]
Since $\max_i c_i \le \|c\|_\infty$ and
$-\min_i c_i \le \|c\|_\infty$, the range satisfies
$(\max_i c_i - \min_i c_i)/2 \le \|c\|_\infty$.
Substituting back,
\[
\|\softmax(b) - \softmax(a)\|_1
\;\le\; \int_0^1 \|c\|_\infty\, dt
\;=\; \|c\|_\infty
\;=\; \|a - b\|_\infty.
\qedhere
\]
\end{proof}

\end{document}